\documentclass[runningheads]{llncs}
\usepackage{booktabs}       
\usepackage{amsfonts}       
\usepackage{nicefrac}       
\usepackage{microtype}      
\usepackage{comment}
\usepackage{bbm}
\usepackage{algorithm}
\usepackage{algorithmic}
\usepackage{multirow}
\usepackage{siunitx,etoolbox}
\usepackage{float}
\usepackage{enumitem}
\usepackage{adjustbox}
\usepackage{subfigure}
\usepackage{booktabs} 
\usepackage[numbers]{natbib}

\usepackage{amsmath}
\usepackage{amssymb}
\usepackage{bm}
\usepackage{graphicx}
\usepackage{mathtools}
\usepackage{stmaryrd}
\usepackage{capt-of}
\usepackage{color}
\usepackage[title]{appendix}
\usepackage[symbol]{footmisc}

\usepackage{tikz}
\usetikzlibrary{shapes.geometric, arrows, positioning, decorations.markings}
\tikzstyle{state} = [rectangle, draw=black,align=center, minimum height=2em]
\tikzstyle{whiteArrow} = [thick, decoration={markings,mark=at position
   1 with {\arrow[semithick]{open triangle 60}}},
   double distance=1.4pt, shorten >= 5.5pt,
   preaction = {decorate},
   postaction = {draw,line width=1.4pt, white,shorten >= 4.5pt}]
\tikzstyle{blackArrow} = [thick,->,>=stealth, line width=0.5mm]
\tikzstyle{arrow} = [thick,->,>=stealth]

\newtheorem{assumption}{Assumption}

\def\th{{^{th}}}
\newcommand{\vect}[1]{\mathbf{\bm{#1}}} 
\newcommand{\mat}[1]{\mathbf{\bm{#1}}} 
\newcommand{\set}[1]{\left\{#1\right\}}
\newcommand{\instS}{\mathcal{X}} 
\newcommand{\classNum}{K} 
\newcommand{\hypo}{h} 
\newcommand{\distr}{\mathbb{P}} 
\newcommand{\ind}[1]{\left\llbracket#1\right\rrbracket} 
\newcommand{\range}[1]{\left[#1\right]} 
\newcommand{\dprod}[1]{\left\langle #1\right\rangle} 
\newcommand{\marg}{\eta}
\newcommand{\conf}{C} 
\newcommand{\loss}{L} 
\newcommand{\sConf}{\widehat{\conf}} 
\newcommand{\perf}{\Psi} 
\newcommand{\util}{\mathcal{U}} 
\newcommand{\expect}[1]{\mathbb{E}\left[#1\right]} 
\newcommand{\expectx}[2]{\mathbb{E}_{#1}\left[#2\right]} 
\newcommand{\prob}[1]{\Pr\left(#1\right)} 
\newcommand{\norm}[1]{\left\lVert#1\right\rVert}

\DeclareMathOperator{\grad}{\nabla} 
\DeclareMathOperator*{\argmax}{argmax}
\DeclareMathOperator*{\argmin}{argmin}
\DeclarePairedDelimiter{\abs}{\lvert}{\rvert}

\renewrobustcmd{\bfseries}{\fontseries{b}\selectfont}
\newrobustcmd{\B}{\bfseries}

\newcommand{\bh}{\mathbf{h}}
\newcommand{\cC}{\mathcal{C}}
\newcommand{\cH}{\mathcal{H}}
\newcommand{\cU}{\mathcal{U}}
\newcommand{\bC}{\mat{C}}
\newcommand{\bP}{\mathbb{P}}
\newcommand{\bR}{\mathbb{R}}

\newcommand{\bas}[1]{\begin{align*}#1\end{align*}}

\newcommand{\bfone}{\bm{1}}
\makeatletter
\newcommand*\bigcdot{\mathpalette\bigcdot@{1}}
\newcommand*\bigcdot@[2]{\mathbin{\vcenter{\hbox{\scalebox{#2}{$\m@th#1\bullet$}}}}}
\makeatother
\newcommand{\diag}{\text{diag}}

\begin{document}

\title{Consistent Classification with Generalized Metrics}

\author{Xiaoyan Wang\inst{1}\thanks{X. Wang and and R. Li contributed equally. X. Wang and and R. Li completed research while students at Illinois. B. Yan completed research while a student at UT Austin.} \and
Ran Li\inst{2} \and
Bowei Yan\inst{3} \and
Oluwasanmi Koyejo\inst{1,2}}

\authorrunning{X. Wang et al.}

\institute{Department of Computer Science, University of Illinois at Urbana-Champaign\\
\email{xiaoyan5@illinois.edu, sanmi.koyejo@gmail.com}\\
\and
Google Inc.
\email{ryannli1129@gmail.com}
\and
Independent Researcher
\email{yanbowei@gmail.com}}
\maketitle              

\begin{abstract}
We propose a framework for constructing and analyzing multiclass and multioutput classification metrics i.e., involving multiple, possibly correlated multiclass labels. Our analysis reveals novel insights on the geometry of feasible confusion tensors -- including necessary and sufficient conditions for the equivalence between optimizing an arbitrary non-decomposable metric and learning a weighted classifier.
Further, we analyze averaging methodologies commonly used to compute multioutput metrics and characterize the corresponding Bayes optimal classifiers. We show that the plug-in estimator based on this characterization is consistent and is easily implemented as a post-processing rule. Empirical results on synthetic and benchmark datasets support the theoretical findings.

\end{abstract}

\section{Introduction}
\label{introduction}

Learning with weighted losses is known to be a population optimal classification strategy (equiv. Bayes optimal) for a wide variety of performance metrics. For instance, weighted losses can be used to estimate Bayes optimal classifiers for binary classification with linear or fractional-linear metrics such as weighted accuracy and F-measure~\citep{NIPS2014_5454}, multiclass classification with linear, concave or fractional linear metrics such as weighed accuracy and ordinal loss~\citep{narasimhan2015consistent}, and multilabel classification with averaged linear and fractional-linear metrics~\citep{NIPS2015_5883}.

Perhaps due to these theoretical results and evident practical success, learning using weighted losses is a popular strategy for constructing predictive models when attempting to optimize complex classification metrics. Unfortunately, it is not known in general when the weighted classifier strategy is a convenient heuristic vs. when it results in provably consistent classifiers. This gap in the literature motivates the question, { {\it when is classification with weighted losses provably consistent?}} We provide an answer by characterizing necessary and sufficient conditions under which learning with weighted losses can recover population optimal classifiers. Interestingly, our results justify the use of weighted losses for many practical settings. For instance, we recover known results that monotonic metrics and fractional-linear metrics satisfy the necessary and sufficient conditions, and thus are optimized by the weighted classifier.

Beyond binary, multiclass and multilabel classification, multioutput learning, also variously known as multi-target, multi-objective, multi-dimensional learning, is the supervised learning problem where each instance is associated with multiple target variables\citep{Read:2014eu,Saha:2015ij}. Formalizing predictive problems in this way has led to empirical success in applied areas like natural language processing~\cite{rubin2012statistical} and computer vision~\cite{kendall2017multi,zhang2014facial}, where combining different tasks boosts the performance of each individual class.
For example, in a movie recommendation system, the learner must predict discrete user ratings for multiple movies simultaneously. In natural language processing, one can learn the POS tagging, chunking, and dependency parsing jointly~\cite{hashimoto2016joint}. In this manuscript, we aim to provide a theoretical understanding of multioutput classification problems i.e. where all outputs are discrete.

In particular, Bayes optimal and consistent classifiers for multioutput classifiers have so far remained unexplored. To this end, another goal of this manuscript is to characterize Bayes optimal multioutput classifiers for a broad range of metrics.

Perhaps the most popular approach for constructing multioutput classifiers in practice is by averaging multiclass metrics. Interestingly, this mirrors the popularity of averaged binary metrics for multilabel classification~\citep{NIPS2015_5883}. Averaged multiclass metrics are constructed by averaging with respect to examples separately for each output (macro-averaging), or with respect to both outputs and examples (micro-averaging). For such averaged metrics, when the classifier is given by a function of the confusion matrix, we characterize both necessary and sufficient conditions for the Bayes optimal classifier to be given by a simple weighted classifier -- specifically, the deterministic classifier which minimizes a weighted loss. We note that this result holds even when the outputs are highly correlated. Further, we show that the associated weights are shared by all the outputs when the metric is micro-averaged. Taken together, these results clarify the role of output correlations in averaged multioutput classification.

The family of fractional linear metrics is of special interest, as examples in this family include widely used metrics such as the multioutput averaged F-measure, among others. For fractional linear metrics, we propose a simple plug-in estimator that can be implemented as a post-processing rule (equiv. as a weighted classifier). We show that this plug-in classifier is consistent i.e. the population utility of the empirical estimator approaches the utility of the Bayes classifier with large samples. We also present experimental evaluation on synthetic and real-world benchmark datasets and a movie recommendation dataset comparing different estimation algorithms. The empirical results show that the proposed approach leads to improved classifiers in practice.

\paragraph{\bf Main contributions:}

\begin{enumerate}[noitemsep]
    \item We characterize necessary and sufficient conditions which determine when weighted classification results in a Bayes optimal multioutput classification. Our characterization recovers recent results on sufficient conditions for binary, multiclass, and multilabel classification.

    \item We show that even when labels are correlated, under standard assumptions, the Bayes optimal multioutput classifier decomposes across outputs.

    \item We propose a plug-in estimator for averaging of fractional-linear class performance metrics, and provide a thorough empirical evaluation. Further, we empirically analyze conditions where using the Bayes optimal procedure may be helpful and other cases where its use may not affect performance -- thus providing practical guidance.

\end{enumerate}

\vspace{-0.2cm}
\subsection{Related Work}
Perhaps due to increasing applied interest in complex classification metrics for specialized applications, there is a growing literature on the analysis and practical implementation of consistent classifiers. ~\citet{NIPS2014_5454} discuss consistent binary classifiers for generic ratios of linear metrics. ~\citet{tewari2007consistency} showed that multiclass classifiers constructed using consistent binary classifiers may still lead to inconsistent multiclass results. ~\citet{narasimhan2015consistent} further propose consistent multiclass classifiers for both concave and fractional-linear metrics. ~\citet{osokin2017structured} consider structured prediction with convex surrogate losses for decompose multiclass classification metrics (i.e. metrics that can be expressed as an average over samples).

Studies of multilabel classification~\citep{dembczynski2012label, dembczynski2010bayes} compare separate (label-wise independent) classification to a variety of correlated label approaches for optimizing hamming loss, showing that separate classification is often competitive. ~\citet{dembczynski2010regret} also critically highlight how the same multilabel classifier may not be optimal for different loss functions -- thus, classifiers must be appropriately tuned to metrics of interest. \citet{NIPS2015_5883} reveal a parametric form for population optimal multilabel classifiers which can be decomposed to binary classifiers and explore efficient algorithms for fractional of linear metrics. We note that  consistent binary, multiclass and multilabel classification are special cases of consistent multioutput classification.

While the multioutput problem is ubiquitous, much of the literature focuses on algorithms and applications -- and few (if any) prior work has considered consistency to the best of our knowledge.

One line of work proposes new algorithms for multioutput problems, most of which are designed to model the correlation relationships between the labels. Examples include the Bayesian chain classifier~\cite{zhang2010multi,zaragoza2011bayesian}, classifier trellises~\cite{read2015scalable}, and general graphical models~\cite{rubin2012statistical} among others.

\citet{BIELZA2011705} address the multioutput classification problem using what they call the multidimensional Bayesian network classifiers (MBCs). However, the only metric they consider is the 0-1 loss.~\citet{BORCHANI20121175} further propose a Markov blanket-based approach to train the MBCs from multioutput data.~\citet{Saha:2015ij} formulate the multitask multilabel framework to make predictions on data from Electronic Medical Records, and solve the iterative optimization problem by block co-ordinate descent.~\citet{Read:2014eu} solve the multi-dimensional (multioutput) classification problem by modeling label dependencies.


\section{Problem Setup and Notation}
\label{framework}

Consider the multioutput classification problem where $\mathcal{X}$ denotes the instance space and $\mathcal{Y} = \left[K\right]^M $ denotes the output space with $M$ outputs and $K$ classes per output. Without loss of generality, we assume that the number of classes is the same for all outputs, i.e. $K_m=K$ for $m = 1,\dots,M$. If the numbers of classes are different for each output, we can simply set $K = \max_m K_m$ i.e. padding classes with null labels as required.
We provide several such examples in practice in Section~\ref{experiment}.
Assume the instances and outputs follow some probability distribution $\mathbb{P}$ over the domain $\mathcal{X}\times\mathcal{Y}$.
A dataset is given by $N$ samples $(x^{(i)}, \mathbf{y}^{(i)}) \overset{\text{i.i.d}}{\sim} \mathbb{P}, i\in[N]$. Since $\mathbb{P}$ is general, the outputs could be highly correlated across outputs.

Define the set of randomized classifiers $\mathcal{H}_r=\{\bh:  \mathcal{X} \to (\Delta^K)^M \}$, where $\Delta^q = \set{\mathbf{p}\in [0,1]^q: \sum_{i=1}^q p_i = 1}$ is the $q-1$ dimensional probability simplex.
For any multioutput classifier, we can define the confusion tensor as follows.
Assume $\bh \in \cH_r$, and the prediction for the $m$th output is $h_m(\cdot) \in \bR^K$. Let $\vect{\eta}(x) \in \bR^{K\times M}$ denote the marginal class probability for any given instance $x$, whose $(k,m)$th element is the conditional probability of output $m$ belonging to class $k$: $\eta^m_k(x) = \bP(Y_m = k \mid X = x)$.
The population confusion tensor is $\mat{C}\in [0,1]^{M\times K\times K}$, with elements are defined as
\begin{align}
\mat{C}_{m,k,\ell} = \int_x h^m_k(x)\vect{\eta}^m_\ell(x)d\bP(x),
\label{eq:defC}
\end{align}
or equivalently, $\mat{C}_{m,k,\ell} = \bP( h^m_k(x)=1, Y_m=\ell)$.

The sample confusion tensor is defined as $\mat{\sConf}(\vect{h}) = \frac{1}{N} \sum_{n=1}^N \mat{\sConf}^{(n)}(\vect{h})$, where
$\mat{\sConf}^{(n)}(\vect{h}) \in \{ 0, 1 \}^{M\times K\times K}$, and $\sConf_{m,i,j}^{(n)}(\vect{h}) = \ind{y_m^{(n)}=i, \hypo_m(x^{(n)})=j}$. Here $\ind{\cdot}$ is the indicator function, so $\sum_{i=1}^{K}\sum_{j=1}^{K}\sConf_{m,i,j}^{(n)}(\vect{h})=1$.
For $\bh\in \cH_r$, due to the linearity of confusion tensor definition, we have $\mat{\sConf}^{(n)}(\vect{h}) \in \range{ 0, 1}^{M\times K\times K}$.
Note that for multiclass classification with a single output, the confusion tensor reduces to a $K\times K$ matrix, commonly simply known as the confusion matrix.

\paragraph{Performance Metrics}
We consider the general class of performance metrics $\cU_\Psi: \mathcal{H} \mapsto \mathbb{R}_+$ for multioutput problems, which can be represented as a function of the confusion tensor, i.e. $\cU_\Psi(\mathbf{h}) = \Psi(\mat{C}(\mathbf{h}))$. This setting has been studied in binary classification~\cite{pmlr-v80-yan18b}, multiclass classification~\cite{narasimhan2015consistent} and multilabel classification~\cite{koyejo2015consistent}.

The goal is to learn the Bayes classifier with respect to the given metric:
$$\vect{h}^*_{\Psi} \in \argmax_{\vect{h}}\; \mathcal{U}_\Psi(\vect{h}). $$
We denote the optimal utility as $\mathcal{U}^*_{\Psi} = \mathcal{U}_\Psi(\vect{h}^*_{\Psi})$.
We say a classifier $\vect{h}_N$ constructed using finite data of size $N$ is $\Psi$-consistent if $ \mathcal{U}_\Psi(\vect{h_N})\xrightarrow{ \mathbb{P} }\mathcal{U}^*_{\Psi}$. And to measure the non-asymptotic performance of a learned classifier, we define regret as follows.
\begin{definition}
\label{psi-regret}
($\perf$-regret). For any classifier $\mathbf{h}$ and a function $\perf$: $\left[0,1\right]^{M\times K \times K} \to \mathbb{R}_+$, define a $\perf$-regret of $\mathbf{h}$ w.r.t. distribution $\mathbb{P}$ as the difference between its $\perf$-performance and the optimal:
$\util^{*}_{\perf, \mathbb{P}} - \util_{\perf,\mathbb{P}}(\mathbf{h}).$
\end{definition}

\paragraph{Notation}
Throughout the paper, we use uppercased bold letters to represent tensors and matrices, and lowercased bold letters to represent vectors. Let $e_i$ represent the $i$th standard basis whose $i$th dimension is 1 and 0 otherwise $e_i=(0,\cdots,1,\cdots,0)$. We follow the tensor computation notation of~\citet{kolda2009tensor}.
For an order-$P$ tensor $\mat{A}\in \bR^{D_1\times\cdots\times D_P}$, we use $A_{i_1, \dots, i_P}$ to represent its elements, where $i_p\in [D_p]$ for $p = 1, \dots, P$; and $\mat{A}_{\cdot, \cdot, \cdots, i_p, \cdots, \cdot}$ to represent the $i_p$th mode-$D$ slice of $\mat{A}$, which is a $P-1$ dimensional tensor obtained by fixing the $p$th index to be $i_p$. For two tensors $\mat{A}, \mat{B}$ of the same dimension, we define their inner product as the sum of element-wise product over all corresponding positions, {\it i.e.}, $\dprod{\mat{A}, \mat{B}} = \sum_{i_1=1}^{D_1}\cdots\sum_{i_P=1}^{N_P} A_{i_1, \cdots, i_P}B_{i_1, \cdots, i_P}$. The $D$-mode (vector) product of a tensor $\mat{A}\in\bR^{D_1,\cdots,D_{j-1}, D_j, D_{j+1}\cdots,D_P}$ with a vector $\mat{u}\in \bR^{D_j}$ is denoted by $A\bigcdot_j \mat{u}$. The result is of order $P-1$ with size $D_1\times \cdots \times D_{j-1}\times D_{j+1}\times \cdots \times D_P$. Element-wise,
\bas{
(A\bigcdot_j \mat{u})_{i_1, \cdots, i_{j-1},i_{j+1},\cdots, i_P} = \sum_{k=1}^{D_j} A_{i_1,\cdots, i_{j-1},k,i_{j+1}, \cdots, i_P}\mat{u}_k
}
We use $\otimes$ to represent the outer product between vectors or tensors.

\section{Bayes Optimal Multioutput Classifiers}
\label{bayes}

In this section, we characterize the conditions for defining a Bayes optimal classifier under general performance metrics.

\subsection{Properties of Confusion Tensors}
\label{sec:conf_prop}
We consider the properties of the classifiers in the confusion tensor space. First, we define the set of feasible confusion tensors for all feasible classifiers.

\begin{definition}[Feasible Confusions]
  Given the distribution $\mathbb{P}$, Let $\mat{C}(\vect{h}) = \mathbb{E}{[\hat{\mat{C}}(\vect{h})]}$ denote the population confusion. The set of all feasible population confusions is given by:
  $\mathcal{C} = \set{\mat{C}(\vect{h}) \mid \bh \in \cH_r}$.
\end{definition}

By the linearity of the confusion tensor, we have the following.
\begin{lemma}[Convexity and compactness]
\label{lemma:ConvexCompact}
  The set $\mathcal{C}$ is convex and compact.
\end{lemma}

The existence of an optimal classifier follows from the compactness and the convexity of $\cC$. The Bayes confusion matrix corresponding to the utility $\cU_\Psi(\mathbf{h})$ is denoted by $\mat{C}^* = \mat{C}(\vect{h}^*_{\Psi})= \max_{\bC} \Psi(\bC)$.
 Moreover, we have the following property of the optimal confusion tensor.

\begin{lemma}
  \label{lem:supporting}
  Let $\mathcal{C}$ be the set of feasible confusions with boundary $\partial \mathcal{C}$.
  If $\mat{C}^*\in \partial \cC$, then there exists $\mat{L} \in \bR^{M\times K \times K}$ such that $\mat{C}^* \in \arg\min_{\bC} \dprod{ \mat{L}, \mat{C}}$.
\end{lemma}

In other words, optimizing any metric that satisfies $\mat{C}^*\in \partial \cC$ can be reduced in maximizing a weighed loss. Thus, we state that any utility which satisfies $\mat{C}^* \in \arg\min_{\bC} \dprod{ \mat{L}, \mat{C}}$ for some $\mat{L}$ admits a weighted Bayes optimal (we will define this formally in the sequel). Importantly, we note that $\mat{L}$ is not unique, since the optimization is invariant to global scale and global additive constants. In the sequel, will usually assume $\norm{\mat{L}}=1$ for some norm $\norm{\cdot}$.

On the other hand, we claim that if the Bayes optimal follows a weighted form, then the optimal confusion tensor necessarily lies on the boundary of the feasible set.

\begin{lemma}
  \label{lemma:only}
  If the utility function admits a weighted Bayes optimal i.e. $\exists \, \mat{L}$ such that  $\bC^* \in \arg\min_{\bC} \dprod{ \mat{L}, \mat{C}}$ then the Bayes optimal confusion $\mathbf{C}^* \in \partial \mathcal{C}$.
\end{lemma}

Taken together, Lemma~\ref{lem:supporting} and Lemma~\ref{lemma:only} characterize necessary and sufficient conditions for the Bayes optimality of weighted classifiers. For specific classes of metrics, the loss tensor can be characterized in closed-form.

\begin{definition}[Monotonic Metrics]\label{monotinic}
  A metric $\Psi$ is strictly monotonic if $\, \forall \, m \in [M]$, $\Psi$ is non-decreasing with respect to all elements of $\{ \mat{C}_{m,i,i} \mid i \in [K] \}$ and non-increasing with respect to all elements of $\{\mat{C}_{m,i,j} \mid i, j \in [K]\}$.
\end{definition}

Monotonic metrics are ubiquitous, as they capture the intuition that the utility should reward better performance and penalize worse performance (as measured by the confusion entries). To our knowledge, all classification metrics in common use satisfy monotonicity. For monotonic and differentiable metrics, we can characterize the loss tensor by the negative gradient of $\Psi$ at the optimal confusion tensor.

\begin{lemma}[Optimal confusion tensor for monotonic and differentiable metric function]
\label{lem:gradient_support}
Let $\Psi$ be a monotonic and differentiable metric, and $\mat{C}^* = \mat{C}(\bh^*_{\Psi})$.
Then $\mat{C}^* \in \max_{\bC} \dprod{ \nabla \Psi(\bC^*), \mat{C}} $.
\end{lemma}

Thus, without loss of generality, we can fix the loss matrix for monotonic and differentiable metrics as $\mat{L} = 1-\nabla \Psi(\bC^*)$.
Note that $\nabla \Psi(\bC^*)$ does not depend on $\bC^*$ for weighted losses such as 0-1 loss (i.e. accuracy), thus can be calculated in closed form. While the Bayes optimal is a weighted classifier, obtaining the loss tensor is sometimes non-trivial. We propose an iterative algorithm in Section~\ref{algorithm}.

Finally, we also note that as a straightforward consequence of convexity of the feasible set, every confusion matrix can be computed as a mixture of two boundary points; thus all Bayes optimal classifiers can be expressed as a mixture of two weighted classifiers. This straightforward corollary is stated more formally without proof.
\begin{corollary}[All Bayes Optima]
  \label{lemma:all}
  Any Bayes optimal confusion $\mathbf{C}^* \in \mathcal{C}$ can be expressed as the mixture of two weighted confusion matrices $\mathbf{C}^* = \alpha\mathbf{C}_1 + (1-\alpha)\mathbf{C}_2,$ where $\alpha \in [0, 1]$ and  $\bC_i \in \arg\min_{\bC} \dprod{ \mat{L}_i, \mat{C}}$ for some $\{\mat{L}_1, \mat{L}_2\}.$
\end{corollary}

\subsection{Weighted Bayes Optimal Classifiers}
The results in Section~\ref{sec:conf_prop} are concerned only with the confusion tensor, and do not enforce any assumption on the classifier or the data distribution.
When the joint distribution of the data is well-behaved, we can extend the optimization over confusion tensors into optimization over the corresponding classifiers.
We introduce the following assumption on the joint distribution of the data.
\begin{assumption}\label{assumption1}
Assume $\mathbb{P}(\{ \vect{\eta}^m(X) = \mathbf{c} \}) = 0 \; \forall \mathbf{c} \in \Delta^K, \; m \in [M]$.
Furthermore, let $Z^m= \vect{\eta}^m(X)$ with density $p_{\eta}(Z^m)$. For all $m \in [M]$, $p_{\eta}(Z^m)$ is absolutely continuous with respect to the Lebesgue measure restricted to $\Delta^K$.
\label{ass:data}
\end{assumption}
Analogous regularity assumptions are widely employed in literature on designing well-defined complex classification metrics and seem to be unavoidable (we refer interested reader to~\cite{pmlr-v80-yan18b,narasimhan2015consistent} for details).
The Bayes optimal classifier for linear multioutput metrics takes a particularly simple form.
\begin{definition}
\label{def:Min-Form}
We say a metric $\cU_\Psi(\mathbf{h})$ admits a weighted Bayes optimal classifier $\bh^*$ if there exists a loss tensor $\mat{L}\in \bR^{M\times K\times K}$, such that  $\bh^{*m}(x)=e_{\ind{ k=\arg\min_k \dprod{\mat{L}_{k\cdot m}, \mat{\eta}^m} }}.$
\end{definition}

\begin{theorem}
\label{thrm:Min-Form}
Under Assumption~\ref{ass:data}, when the metric $\cU_\Psi(\mathbf{h})$ admits a weighted Bayes optimal (Definition~\ref{def:Min-Form}), $\exists$ a Bayes optimal classifier $\bh^*$ which satisfies $\bh^{*m}(x)=e_{\ind{ k=\arg\min_k \dprod{\mat{L}_{k\cdot m}, \mat{\eta}^m} }}$, where $\mat{L}$ is as defined in Lemma~\ref{lem:supporting}.
\label{thm:max_over_eta}
\end{theorem}
Theorem~\ref{thm:max_over_eta} states that the correlation between outputs are fully reflected in the conditional probability. It unifies some known results: when $M=1$, this recovers the multiclass optimal result in~\cite{narasimhan2015consistent}; when $K=2$, this recovers the multilabel results studied in~\cite{koyejo2015consistent}; when $M=1,K=2$, one can show that the weighted classifier reduces to standard thresholding for binary classification~\cite{pmlr-v80-yan18b,NIPS2014_5454}.

Interestingly, Theorem~\ref{thrm:Min-Form} combined with Corollary~\ref{lemma:all} suggests a remarkable simplicity of all Bayes optimal classifiers, namely that either there exists a deterministic Bayes classifier, or there exists a Bayes classifier given by a mixture of two deterministic classifiers. This observation is stated more formally in the following corollary. 
\begin{corollary}[All Bayes Classifiers]
	\label{lemma:all}
	Under Assumption~\ref{ass:data}, for any metric $\cU_\Psi(\mathbf{h})$, $\exists$ a Bayes optimal classifier $\bh^*$ which satisfies $\mathbf{h}^* = \alpha\mathbf{h}_1 + (1-\alpha)\mathbf{h}_2$, where $\mathbf{h}_i$ are weighted deterministic classifiers (Theorem \ref{thrm:Min-Form}), and $\mathbf{C}_1 =\mathbf{C}(\mathbf{h}_1)$, $\mathbf{C}_2 = \mathbf{C}(\mathbf{h}_2)$, $\mathbf{C}(\mathbf{h^*}) = \mathbf{C}^* = \alpha\mathbf{C}_1 + (1-\alpha)\mathbf{C}_2$ as defined in Corollary~\ref{lemma:all}.
\end{corollary}
The theorem is a straightforward consequence of Corollary~\ref{lemma:all}, and the observation that by definition, $\mathbf{C}_i$ admit a weighted Bayes optimal.
Notably the same mixture weight $\alpha$ is optimal point-wise. The special case of deterministic classifiers (Theorem~\ref{thrm:Min-Form}) is recovered when $\alpha=0$.

\subsection{Averaged Multioutput Metrics and their Bayes Optima}
\label{sec:averaging}
The most common technique for constructing multioutput metrics is by treating each output as separate multiclass problem and averaging the corresponding multiclass performance metrics. Here we distinguish two types of averaging and discuss differences in their population behavior. For both averaging methods, we assume there exists a multiclass performance metric $\psi: [0,1]^{K\times K}\to \bR$, and an output-specific weight vector $\mathbb{\alpha} \in \bR^M$.

\paragraph{Microaveraging}
Micro-averaging is implemented by averaging the multiclass confusion matrices for each output, then applying the performance function on the averaged confusion matrix. Formally,

\bas{
\Psi_{\text{micro}}(\mat{ \hat{C} }) = \psi \left( \sum_{m=1}^M \mathbb{\alpha}_m\hat{\bC}_{m \cdot \cdot} \right)
=  \psi \left(  \frac{\mathbb{\alpha}_m}{N}\sum_{m=1}^{M}\sum_{n=1}^{N}\sConf_{m,i,j}^{(n)}(h_m) \right)
}

\paragraph{Macroaveraging}
Macro-averaging is implemented by first applying the performance function $\psi$ to each output confusion matrix, then averaging over the outputs. Formally,
\bas{
\Psi_{\text{macro}}(\mat{ \hat{C} }) = \sum_{m=1}^M \mathbb{\alpha}_m\psi \left( \hat{\bC}_{m \cdot \cdot} \right)
= \sum_{m=1}^M \mathbb{\alpha}_m  \psi \left( \frac{1}{N}\sum_{n=1}^{N}\sConf_{m,i,j}^{(n)}(h_m)  \right).
}
Examples of macro-averaged metrics include Decathlon score used by~\cite{rebuffi2017learning}, where $\psi(C_{m\cdot\cdot}) = 1-\diag(C_{m\cdot\cdot})$. Micro-averaging and Macro-averaging result in identical performance metrics for the case of linear functions $\psi$, such as those used in weighted average accuracy. Yet it will not be surprising these two types of averaging methodology result in different optimal classifiers for general non-linear functions $\psi$.
Examples include linear metrics like Ordinal~\citep{narasimhan2015consistent}, Micro-$F_1$~\citep{kim2013genia} and Macro-$F_1$~\citep{lewis1991evaluating}; weighted accuracy puts an exponential weights to each class; and polynomial functions are used in metrics such as Decathlon score~\citep{rebuffi2017learning}. Table~\ref{table-metrics} lists several examples of $\psi$ commonly used by practitioners in multiclass and multioutput classification.

\begin{table}[h]
\caption{Examples of multiclass performance metrics}
\label{table-metrics}
    \begin{center}
    \begin{tabular}{ccccc}
    \toprule
    {Performance Metrics} & $\psi(\mat{\conf})$ \\
    \midrule
    Ordinal & $\sum_{i=1}^{K}\sum_{j=1}^{K}(1-\frac{1}{n-1} \abs{i-j})\conf_{i,j}$\\[5pt]
    Micro-$F_1$ & $\frac{2\sum_{i=2}^{K}\conf_{i,i}}  {2-\sum_{i=1}^{K}\conf_{1,i}-\sum_{i=1}^{K}\conf_{i,1}}$ \\[10pt]
    Macro-$F_1$ & $\frac{1}{K} \sum_{i=1}^K  \frac{2 \conf_{i,i}} {\sum_{j=1}^{K}\conf_{i,j}+\sum_{j=1}^{K}\conf_{j,i}}$ \\[10pt]
    Weighted$_{\gamma}$ & $\sum_{i=1}^{K} e^{-\gamma i} \conf_{ii}$\\[7pt]
    Min-max & $\min_{i\in\left[K\right]} \frac{C_{i,i}}{\sum^{n}_{j=1}C_{i,j}}$\\
    Polynomial & $ (1-\diag(\mat{C}))^{\gamma}$\\
    \bottomrule
    \end{tabular}
    \end{center}
\end{table}

The following propositions state some differences between the two kinds of averaging.
\begin{proposition}[Shared Loss for micro-averaged metrics]
\label{prop:micro_share_plane}
For micro-averaged multioutput metrics, all outputs (considered as multiclass problems) share the same loss matrix $ \mat{L}^*_{m\cdot \cdot} = \mat{S},\, \forall \, m \in [M]$.
\end{proposition}

Now, for macro-averaged multioutput metrics, the following proposition shows that the Bayes optimal classifier decomposes across the labels.
\begin{proposition}[Decomposability of macro-averaged Bayes optimal]
\label{macro-decompose}
  The macro-averaged utility decomposes as $\mathcal{U}_{\text{macro}}(\vect{h}) = \sum_{m\in [M]} \mathcal{U}_m(h_m)$. Without additional classifier restrictions, the Bayes optimal macro-averaged classifier decomposes as:
  $h_m^*(x) = \underset{h_m}{\argmax} \text{ } \mathcal{U}_m(h_m)$.
\end{proposition}
Proposition~\ref{macro-decompose} states that if there are no constraints on the classifier, then optimizing the macro-averaged metric is equivalent to optimizing each task/output separately. We observe that in real-world multitask learning problems, one often imposes structural assumptions that correlate the outputs to boost the performance.
This highlights an interesting gap between finite sample and population analysis which we leave for future work.

\subsection{Fractional-Linear Multiclass Metrics}

Fractional-Linear metrics are a popular family of classification metrics which include the F-measure.
\begin{definition}[Fractional-Linear Metric Functions]\label{ratio-of-linear}
  $\psi$ is a fractional-linear function of the confusion matrix if it can be expressed as $\psi(\mat{C}) = \frac{\dprod{\mat{A}, \mat{C}}}{\dprod{\mat{B}, \mat{C}}}$, where $\mat{A}, \mat{B}\in\mathbb{R}^{K\times K}$ and $\dprod{\mat{B}, \mat{C}} > 0$ for $\forall\mat{C}\in\mathcal{C}$.
\end{definition}

For ease of exposition, we introduce the following definition.
\begin{definition}[Loss-based multiclass metric]\label{loss-based}
  Let $\mat{\loss}\in\range{0, 1}^{\classNum\times\classNum}$ be a loss matrix, its corresponding loss-based metric is defined as $\psi^{\mat{L}}(\bh) = 1 - \dprod{\mat{L}, \mat{\conf}(\bh)}$.
\end{definition}

We can derive two straightforward corollaries for fractional-linear metrics.

\begin{corollary}[Bayes optimal for micro-averaged fractional-linear metrics]
  \label{l-optimal-1}
  Let $\psi$ be fractional-linear, then the micro-averaged multioutput Bayes optimal classifier is a weighted classifier for each output i.e. $h_m(x)\in\argmin_{k\in\range{K}}\mat{L}_k^{*\intercal}\vect{\marg}_m(x)$.
  Let $\widetilde{\mat{L}}^* = \mathcal{U}^*_{\Psi}\mat{B} - \mat{A}$, and $\mat{L}^*$ be the $\range{0, 1}^{K\times K}$ matrix obtained by scaling and shifting $\widetilde{\mat{L}}^*$, then any classifier that is $\psi^{\mat{L}^*}$-optimal is also $\Psi$-optimal.
\end{corollary}

\begin{corollary}[Bayes optimal for macro-averaged fractional-linear metrics]
  \label{l-optimal-2}
  Let $\psi$ be fractional-linear, then the macro-averaged multioutput Bayes optimal classifier is a weighted classifier. Let $\widetilde{\mat{L}}_m^* = \mathcal{U}^*_{\Psi, m}\mat{B} - \mat{A}$, and $\mat{L}_m^*$ be the $\range{0, 1}^{K\times K}$ matrix obtained by scaling and shifting $\widetilde{\mat{L}}_m^*$, then any classifier that is $\psi^{\mat{L}_m^*}$-optimal is also $\Psi$-optimal for each $m \in [M]$.
\end{corollary}

\section{A Bisection Algorithm For Linear-Fractional Metrics}
\label{algorithm}

The Bayes optimality analysis suggests two strategies for estimating consistent classifiers. Both approaches require an estimate of the loss matrix (or tensor). The first takes trains a weighed classifier as the empirical risk minimizer (ERM) of a weighted loss function e.g. weighted multiclass support vector machine~\cite{yang2007weighted}. The second (plug-in) approach first computes an estimate of the conditional probability $\widehat{\vect{\eta}}^m(X)$ then returns a decision rule based on the weighted optimal of Definition~\ref{def:Min-Form}. We focus on the plug-in estimator for this manuscript, and leave details of the ERM estimator for a longer version of this manuscript. We refer the interested reader to \citep{koyejo2015consistent} for additional discussion of the two approaches to consistent estimators for the special case of binary classification.

For the plug-in estimator, observe that the construction of the decision rule is analogous to post-processing the estimated conditional probabilities in the context of the performance metric of interest. This is in contrast to the default rule which simply predicts the most likely class i.e. $\argmax_{k \in [K]} \widehat{\eta_k}^m(X)$.  Once this loss matrix/tensor is determined, the additional computation required for post-processing is $\mathcal{O}(MK^2)$ for the matrix-vector multiplication, which is further reduced to $\mathcal{O}(MK)$ for the common setting of diagonal (class-specific) weights. This additional computation is negligible for small and medium problem sizes.

\vspace{-0.2cm}
\subsection{Bisection Method for Fractional-Linear Metrics}
\label{section:bisection}

We begin by splitting the training dataset $\mathcal{S}$ into $\mathcal{S}_1$ for $\widehat{\vect{\eta}}^{\mathcal{S}_1}$ probability estimation $\mathcal{S}_2$ for obtaining confusion $\mat{\sConf}^{\mathcal{S}_2}$ to be evaluated. For fractional-linear performance metrics $\Psi(\mat{\conf})=\frac{\dprod{\mat{A}, \mat{C}}}{\dprod{\mat{B}, \mat{C}}}$, the maximization of $\gamma$ satisfying $\max_{\mat{\conf}}\Psi (\mat{\conf}) \geq \gamma$ is equivalent as linear minimization of $\min_{\mat{\conf}}- \dprod{\mat{A}-\gamma \mat{B}, \mat{C}}$. During each iteration $t$, we apply bisection method to find a midpoint $\gamma^{t}$  between the lower bound lower bound $\alpha^t$ and the upper bound $\beta^t$. The loss matrix can be computed as $\widehat{\mat{L}}^t = \gamma^t\mat{B} - \mat{A}$. The performance of the linear form classifier is then computed on $\mathcal{S}_2$, and the lower and upper bounds are updated accordingly. Our proposed approach builds on the \citep{narasimhan2015consistent}, originally proposed for multi-class classification. The flowchart of the Bisection algorithm is described in Figure~\ref{fig:algo-graph} with a detailed discussion in Algorithm~\ref{algo:bisection} in the Appendix. Theorem~\ref{consistency-theorem} in the Appendix shows that the Bisection search plug-in classifier is consistent if $\vect{\eta}$ obtained by $\mathcal{S}_1$ satisfies $\expectx{X}{\norm{\widehat{\vect{\eta}}^m(X) - {\vect{\eta}^m(X)}}_{1}} \to 0 \text{ } \forall m$ when $N\to \infty$.

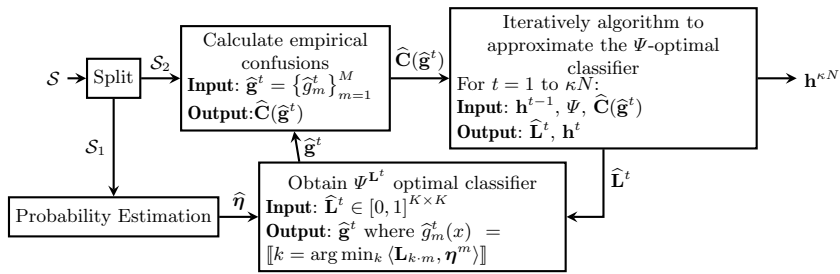
\begin{figure*}[ht]
\begin{center}
\begin{tikzpicture}[thick,scale=0.8, every node/.style={transform shape}]
\node[state](split) {Split};
\node[rectangle, left of = split](input) {$\mathcal{S}$};
\node[state, below = 5em of split](estimate) {Probability Estimation};
\node[state, right = 2em of split, text width = 10em, align=left](confusion){\begin{minipage}{\textwidth}\center Calculate empirical confusions\end{minipage}\\\textbf{Input}: $\vect{\widehat{g}}^t = \set{\widehat{g}^t_m}_{m=1}^M$\\\textbf{Output}:$\mat{\widehat{C}}(\vect{\widehat{g}}^t)$};
\node[state, right = 2em of estimate, text width = 15em, align=left](classifier){\begin{minipage}{\textwidth}\center Obtain $\Psi^{\mat{L}^t}$ optimal classifier\end{minipage}\\\textbf{Input}: $\mat{\widehat{L}}^t\in\range{0, 1}^{K\times K}$\\\textbf{Output}: $\vect{\widehat{g}}^t$ where
${\widehat{g}^t_m(x)}  = \ind{ k=\arg\min_k \dprod{\mat{L}_{k\cdot m}, \mat{\eta}^m} }$};
\node[state, right = 3em of confusion, align=left, text width = 15em](iterate){\begin{minipage}{\textwidth}\center Iteratively algorithm to approximate the $\Psi$-optimal classifier\end{minipage}\\For $t = 1$ to $\kappa N$:\\\textbf{Input}: $\vect{h}^{t-1}$, $\Psi$, $\mat{\widehat{C}}(\vect{\widehat{g}}^t)$\\\textbf{Output}: $\mat{\widehat{L}}^t$, $\vect{h}^t$};
\node[rectangle, right = 2em of iterate](output){$\vect{h}^{\kappa N}$};

\draw [arrow] (input)--(split);
\draw [arrow] (split)-- node [left] {$\mathcal{S}_1$} (estimate);
\draw [arrow] (split)-- node [above] {$\mathcal{S}_2$} (confusion);
\draw [arrow] (estimate) -- node [above] {$\vect{\widehat{\eta}}$} (classifier);
\draw [arrow] (confusion) -- node [above] {$\mat{\widehat{C}}(\vect{\widehat{g}}^t)$} (iterate);
\draw [arrow] (iterate) |- node [right, yshift=2em] {$\mat{\widehat{L}}^t$} (classifier);
\draw [arrow] (classifier.154) -- node [right] {$\vect{\widehat{g}}^t$} (confusion);
\draw [arrow] (iterate) -- (output);
\end{tikzpicture}
\end{center}
\caption{Flow of the multi-output algorithm framework. Note that probability estimation part can be done using any algorithm that estimates scores which approximate the marginal probability $\vect{\widehat{\eta}}(x)\in\range{0, 1}^{M\times K\times K}$ for each sample.}\label{fig:algo-graph}
\end{figure*}


\section{Experiments}
\label{experiment}
We present three different kinds of experimental results. The first are experiments on  synthetic data used to illustrate when the weighted classifier will outperform the default rule, the second set of experiments are on benchmark UCI datasets, and the third is a movie rating prediction task. Note that the proposed procedure can be used to post-process any classifier that estimates probability calibrated scores. To simplify notation, we use the prefix ``C'' for ``consistent'' to denote the post-processed results e.g. C-LogReg to denote consistent (post-processed) logistic regression.

\vspace{-0.2cm}
\subsection{Synthetic Data: Exploring the Advantages of Weighted Classifier}

Our first experiment compares the standard multi-class logistic regression algorithm (LogReg) to the consistent multi-class logistic regression classifier (C-LogReg). Our primary goal in this experiment is to explore and analyze the factors that influence when post-processing will improve performance. We generate 100,000 samples with 10-dimensional features $x$ and 10 classes by standard Gaussian and multinomial distribution. The class probability $\mat{\eta}_k$ for class $k$ is modeled by multinomial logistic regression: $\mat{\eta}_k(x) = \distr(Y_m = k|x) \propto \exp (-w_k^Tx)$.

All experiments use the weighted loss $\Psi^{\text{Weighted}}(\mat{C})=\dprod{\mat{A},\mat{C}}$, where the loss matrix is given by $\mat{L}^{\text{Weighted}} = 1 - \grad  {\Psi(\mat{C})} = 1 - \mat{A}$. Similarly, 0-1 loss takes the form $\mat{L}^{0-1} = 1 - \mat{I}$, where $\mat{I}$ is identity matrix. Our results are presented in the form of a performance ratio
$$\text{PR}_{\text{LogReg}} = \frac {\text{performance of C-LogReg}} {\text{performance of LogReg}}$$ for the following two conditions: (1)
we vary the data generating distribution as $\mat{\eta}_{k}$ by defining $w_{kd} = C_1 \abs{k-d}$, where $C_1$ is the variable, resulting is more or less uniform conditional probabilities; (2) we vary the weight metric as $A_{ii} = e^{-C_2 i}$, where $C_2$ is the variable.

\begin{figure}[ht]
\begin{center}
\centerline{\includegraphics[width=.8\columnwidth]{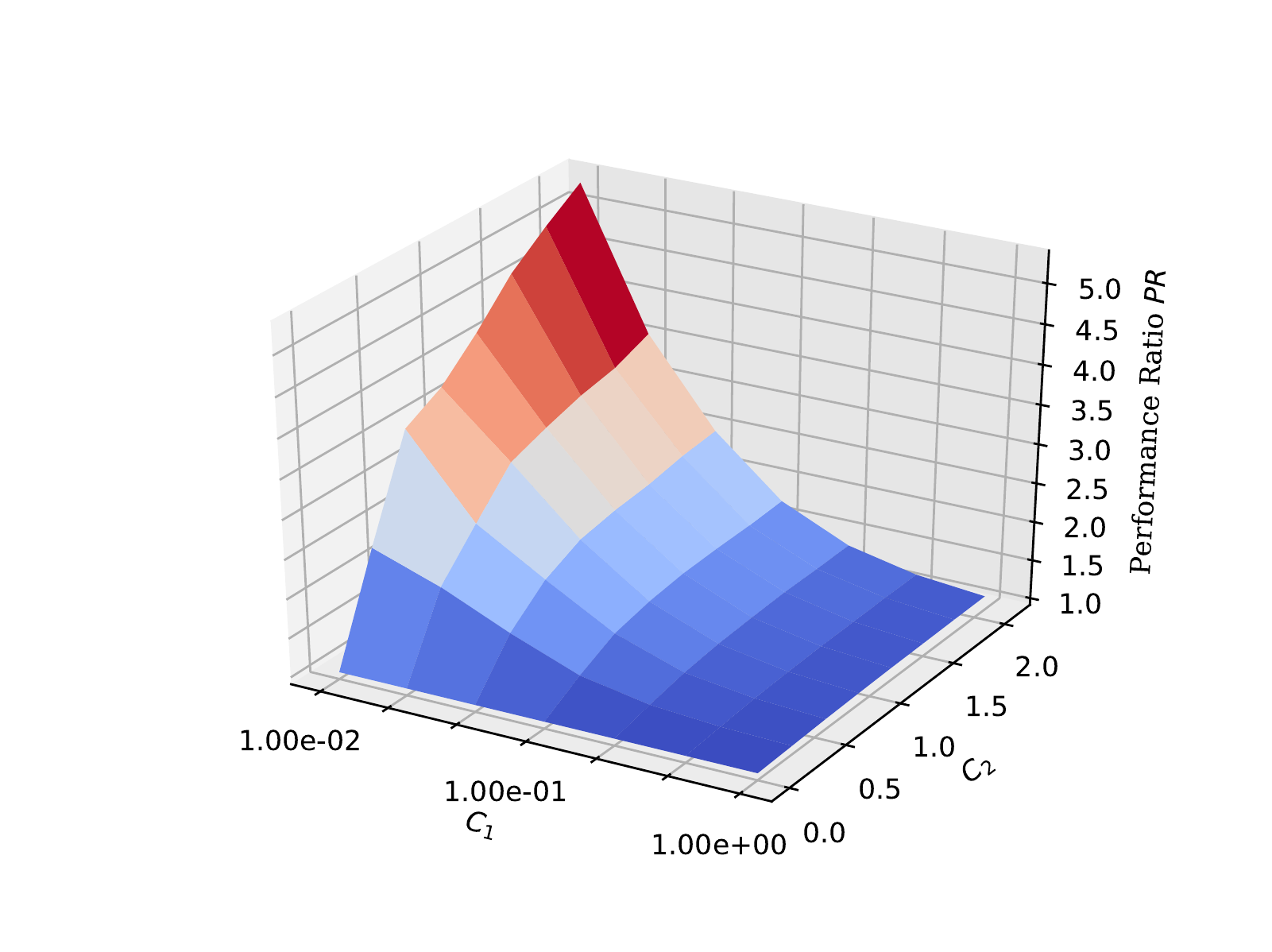}}
\caption{Performance ratio $PR_{\text{LogReg}}$ of synthetic data by $C_1$ and $C_2$ under Weighted performance metric. The largest performance ratio is for low-skew conditional probabilities (close to uniform) and a high-skew weighted metric.}
\end{center}
\label{performace_3d}
\vskip -0.2in
\end{figure}

\begin{table}[ht]
\caption{UCI Datasets used in Experiments}
\label{table-UCI}
    \begin{center}
    \begin{small}
    \begin{sc}
    \begin{tabular}{lcccc}
    \toprule
    Dataset & Instances & Features & Labels & Classes \\
    \midrule
    Car    & 1727 & 5 & 2 & 4\\
    Nursery & 12960 & 7 & 2 & 3 \\
    CMC     & 1473 & 8 & 2 & 4 \\
    Phish   & 1353 & 7 & 3 & 3 \\
    Student & 649  & 30 & 3 & 5 \\
    \bottomrule
    \end{tabular}
    \end{sc}
    \end{small}
    \end{center}
\end{table}

\begin{table}[h]
\caption{Reports of $\text{PR}_{\text{LogReg}}$ (performance ratio of C-LogReg and LogReg) and $\text{PR}_{\text{DT}}$ (performance ratio of consistent Decision Tree classifier C-DT and Decision Tree classifier DT) by micro- and macro-averaging under Ordinal, Micro-F1 and Weighted$_{\bf{\frac{1}{2}}}$ performance metrics. Datasets come from UCI Machine Learning Repository and all results average over 100 iterations with 80\%-20\% train-test split. Consistent algorithms always have better performance. }
\label{table-uci-start}

    \begin{center}
    \begin{small}
    \begin{tabular}{l|ccccc|ccccccc}
    \toprule
    \rule{0pt}{13pt} Dataset & Car & Nursery & CMC & Phish &
    Student & Car & Nursery & CMC & Phish & Student\\[5pt]
    \hline
    &\multicolumn{5}{|c|}{\textnormal{\textbf{Ordinal} metric by \textbf{Micro-averaging}}} & \multicolumn{5}{c}{\textnormal{\textbf{Ordinal} metric by \textbf{Macro-averaging}}}\\
    \hline
    $\text{PR}_{\text{LogReg}}$  & 1.1530 & 1.0885 & 1.0181 & 1.0006 & 1.0121 &
    1.1485 & 1.0903 & 1.0210 & 1.0005 & 1.0103\\
    $\text{PR}_{\text{DT}}$      & 1.1266 & 1.0687 & 1.0214 & 1.0112 & 1.0124 &
    1.1272 & 1.0712 & 1.0212 & 1.0112 & 1.0139\\

    \hline
    &\multicolumn{5}{|c|}{\textnormal{\textbf{Micro-F1} metric by \textbf{Micro-averaging}}} &
    \multicolumn{5}{c}{\textnormal{\textbf{Micro-F1} metric by \textbf{Macro-averaging}}}\\
    \hline
    $\text{PR}_{\text{LogReg}}$  & 1.1972 & 1.0258 & 1.0006 & 1.0001 & 1.1205 &
    1.2091 & 1.0391 & 1.0022 & 1.0090 & 1.4144\\
    $\text{PR}_{\text{DT}}$      & 1.1261 & 1.0625 & 1.0006 & 1.0002 & 1.1915 &
    1.1190 & 1.0899 & 1.0032 & 1.0146 & 1.8712 \\

    \hline
    &\multicolumn{5}{|c|}{\textnormal{\textbf{Weighted$_{\bf{\frac{1}{2}}}$} metric by \textbf{Micro-averaging}}} & \multicolumn{5}{c}{\textnormal{\textbf{Weighted$_{\bf{\frac{1}{2}}}$} metric by \textbf{Macro-averaging}}}
     \\
    \hline
    $\text{PR}_{\text{LogReg}}$  & 2.4311 & 1.6060 & 1.6144 & 1.3942 & 1.1211 &
    2.4177 & 1.6105 & 1.6327 & 1.4055 & 1.1890\\
    $\text{PR}_{\text{DT}}$      & 3.1338 & 1.4815 & 1.6699 & 1.1235 & 1.1112 &
    3.1472 & 1.5057 & 1.7237 & 1.2328 & 1.1012\\
    \bottomrule
    \bottomrule
    \end{tabular}
    \end{small}
    \end{center}
\vskip -0.1in
\end{table}

\begin{table}[ht]
\caption{Comparison of performance on \textbf{Ordinal} metric for OrdRec classifier and consistent OrdRec classifier on dataset MovieLens 100k by \textbf{Micro-averaging} and \textbf{Macro-averaging}.}
\label{table-mf}
    \begin{center}
    \begin{small}
    \begin{sc}
    \begin{tabular}{lcccr}
    \toprule
    Average & OrdRec & C-OrdRec\\
    \midrule
    Micro    & 0.8603$\pm$0.0010 & \B 0.8640$\pm$0.0009\\
    Macro   & 0.8577$\pm$0.0032 & \B 0.8643$\pm$0.0022\\
    \bottomrule
    \end{tabular}
    \end{sc}
    \end{small}
    \end{center}
\end{table}

Figure~\ref{performace_3d} demonstrates the influence of $\mat{C}_1$ and $\mat{C}_2$ on $\text{PR}_{\text{LogReg}}$. For each pair of $\mat{C}_1$ and $\mat{C}_2$, $\text{PR}_{\text{LogReg}}$ is averaged over $30$ iterations with random 80\%-20\% train-test splits. We observe that the consistent classifier works much better than multi-class logistic regression algorithm with smaller $\mat{C}_1$ and larger $\mat{C}_2$ - we can see that in the dark red region, $\text{PR}_{\text{LogReg}}$ is over 3, which means the performance of C-LogReg is more than three times better than LogReg. To understand the trend of $\text{PR}_{\text{LogReg}}$ through $\mat{C}_1$ and $\mat{C}_2$, notice when the two classifier make different decisions, we have $\sum_{n=1}^N \ind {\argmin_{i \in \classNum} \vect{\loss}^{0-1}_j\vect{\eta}_i \neq \argmin_{j \in \classNum} \vect{\loss}^{\text{Weighted}}_j \vect{\eta}_j}$, so the classifiers differ when $\sum_{n=1}^N \ind {\argmax_{i \in \classNum} \vect{\eta}_i \neq \argmax_{j \in \classNum} \vect{A}_j \vect{\eta}_j}$.

Therefore, the larger the value of  $C_1$, the more peaked the class probability $\mat{\eta}_k$ for each class $k$, then the effect of $\mat{\eta}$ becomes more dominant in classification, which results in a smaller difference between two predictions. On the other hand, the larger the value of $C_2$, the more skewed $\mat{A}$ becomes as compared to $\mat{I}$, resulting in larger difference between two predictions. Furthermore, when $C_2=0$, $\mat{A}$ is exactly $\mat{I}$, so we have $PR_{\text{LogReg}}=1$ since C-LogReg and LogReg both optimize 0-1 accuracy. The benefit of post-processing must be compared to the additional computational costs. This experiment provides some guidance on this trade-off.

\vspace{-0.1cm}
\subsection{Benchmark Data: UCI Datasets}
\label{experiment:benchmark}
    We use real-world datasets from UCI repository \citep{Lichman:2013} to evaluate algorithm performances under Ordinal, Micro-F1 and Weighted$_{\frac{1}{2}}$ metrics by micro-averaging and macro-averaging, as shown in Table~\ref{table-metrics}. Table~\ref{table-UCI} presents the information about number of instances, features, labels and classes in each of five benchmark datasets we used.

    The algorithms evaluated are: (1) multi-output logistic regression classifier for each label (LogReg), (2) Random Forest with max depth as 3 (RF), (3) Decision Tree with max depth as 3 (DT), (4) consistent logistic regression (C-LogReg), (5) consistent Decision Tree (C-DT). (1) and (4) are linear and the rest are non-linear classifiers. The hyper-parameters are chosen using double-loop cross-validation. For simplicity, we only report the performance ratios of $\text{PR}_\text{LogReg}$ (performances of C-LogReg over LogReg) and $\text{PR}_\text{DT}$ (performances of C-DT over DT) here. See full original performances, variances and more comparisons of algorithms in the Appendix.

    Table~\ref{table-UCI} presents the information about number of instances, features, labels and classes in each of five benchmark datasets.
    The attributes of each dataset are split into two sets: features and labels. The label assignments are: (1) attributes 1-2 in Car Evaluation (Car), (2) attributes 7-8 in Nursery (Nursery), (3) attributes 7-8 in Contraceptive Method Choice (CMC), (4) attributes 1-3 in Website Phishing (Phish), (5) attributes 26-28 in wiki4HE (wiki4HE). The rest of the attributes are features.

    The performance results, averaged over 100 times with random 80\%-20\% train-test splits, are presented in Table~\ref{table-uci-start}. We notice that $\text{PR}_{\text{LogReg}}$ and $\text{PR}_{\text{DT}}$ keep greater than 1, which means that the consistent algorithms always have better performance under same $\widehat{\vect{\eta}}$: $\Psi(\vect{h}^{\text{C-LogReg}})>\Psi(\vect{h}^{\text{LogReg}})$ and $\Psi(\vect{h}^{\text{C-DT}})>\Psi(\vect{h}^{DT})$. $\text{PR}_{\text{LogReg}}$ and $\text{PR}_{\text{DT}}$ are enlarged specifically under Weighted$_{\frac{1}{2}}$ metric.

\vspace{-0.2cm}
\subsection{Benchmark Data: MovieLens}
For the third experiment, we apply multi-output classification to real world rating prediction.
We use MovieLens 100K Dataset \citep{harper2016movielens} which contains 100,000 tuples of user, movie and the rating of the user on the movie. We convert the dataset to a standard multi-output classification problem by representing the rating matrix as $Y \in [1,2,..K]^{N \times M}$, where $K$ is the number of rating choices $5$ (star $1$ to star $5$), $N$ is the number of users $943$ and $M$ is number of movies $1,682$ for this dataset. $Y^{(n)}_m$ corresponds to the rating of $n-th$ user (sample) $X^{(n)}$ on $m$-th movie (label). We plug in the class probabilities $\mat{\eta}_{k}^{m}(X^{(n)}) = \distr(Y^{m} = k | X^{(n)})$ for sample $X^{(n)}$, label $m$ and class $k$. The distribution $\mat{\eta}$ is derived using the Ordinal regression model OrdRec \citep{koren2011ordrec}, already shown to perform well for ordinal regression-based prediction.

Once the probabilities are estimated, OrdRec predicts the most likely rating class for each user movie pair.

 Since each user only rates a subset of movies, the label space is sparse. We assume that the missing labels are missing at random. Micro-averaging in this setting is equivalent to
$\hat{C}_{i,j}(\vect{h}) = \frac{1}{|\Omega|}\sum_{(m, n) \in \Omega}^{M} \hat{C}_{m,i,j}^{(n)}(\vect{h})$, where $\Omega = \set{(m, n) \mid \text{$y_{m}^{(n)}$ is observed}}$ is the set of observed entries. In Table~\ref{table-mf}, we report the micro-averaged and macro-averaged under Ordinal metric on MovieLens dataset of OrdRec and Consistent-OrdRec classifier (C-OrdRec). The results are averaged over 30 times with 50\%-50\% train-test split. We observe that under all averages, the C-OrdRec classifier results in better performance than OrdRec.

\section{Conclusion}
\label{conclusion}
We outline necessary and sufficient conditions for Bayes optimal multioutput classification using weighted classifiers -- which recovers binary, multiclass and multilabel classification as special cases. We further consider multi-output classification under generalized performance metrics with micro- and macro-averaging, and propose a provably consistent bisection-search classifier for fractional-linear metrics. In a variety of experiments, we find that the proposed estimator can significantly improve performance in practice.

\renewcommand{\bibsection}{\section*{References}}
\bibliography{momc}
\bibliographystyle{plainnat}

\clearpage
\appendix

\section{Proofs of Weighted Classifier Representation}
In this section, we provide the proofs for the theoretical results in the main paper. For ease of navigation, we summarize the notation used in the sequel in Table~\ref{tab:notation}.
\begin{table*}[h]
  \caption{Notation used in paper}
  \label{tab:notation}
  \begin{center}
    \begin{tabular}{ll}
      \toprule
      Symbol  & Description\\
      \midrule
      $N$     & number of instances \\
      $M$     & number of outputs \\
      $K$     & number of classes \\
      $\mat{C}\in [0,1]^{K\times M \times M}$     & confusion tensor \\
      $\mathcal{U} = \Psi(\mat{C})$       & utility of a classifier \\
      $\Delta_q$            & $\set{\mathbf{p}\in [0,1]^q: \sum_{i=1}^q p_i = 1}$\\
      $\mathcal{H}_r $          & set of randomized classifiers $\{\bh:  \mathcal{X} \to \Delta_K^{M}\}$\\
      $\vect{h}$              & multi-output classifier in $\mathcal{H}_d $ or $\mathcal{H}_r $\\
      $\vect{h}^*_{\Psi} \in \argmax_{\vect{h}}\; \mathcal{U}(\vect{h}). $ & Bayes optimal classifier with respect to performance metric $\Psi$.\\
      $\ind{\cdot}$             &  indicator function\\
      $\left[q\right]$            & $\set{1, 2, \cdots, q}$ for all $q \in \mathbb{Z}_+$\\
      $\dprod{\mat{A}, \mat{B}}$      & $\sum_{i_1, \cdots, i_M} A_{i_1,\cdots, i_M} B_{i_1, \cdots, i_M}$\\
      $\mathbb{\eta}^m_k(x) = P(Y_m=k | X=x)$ & conditional probability for $m$th output and class $k$.\\
      $\mathcal{C} = \set{\mat{C}(\vect{h}) \mid \vect{h}\in \cH_r}$ & set of feasible confusion tensors.\\
      $e_i$               & $i$th standard basis whose $i$th dimension is 1 and 0 otherwise $e_i=(0,\cdots,1,\cdots,0)$\\
      $\bm{1}_M\in\bR^M$        & all one vector of dimension $M$\\
      $ \bm{v} \otimes \bm{w}$    & outer product, $(\bm{v} \otimes \bm{w})_{ij}=v_iw_j$\\
      \bottomrule
    \end{tabular}
  \end{center}
\end{table*}

\subsection{Proof of Lemma~\ref{lemma:ConvexCompact}}
Observe that $\mathcal{H}$ is equivalent to the space of vector $L_{\infty}$ functions which is a compact function space.
  \begin{itemize}
  \item[-] Compact: Compactness of $\mathcal{C}$ follows from compactness of $\mathcal{H}$, since the mapping $C: \mathcal{H} \mapsto [0, 1]^{M \times K \times K}$ is linear and bounded.

\item[-] Convex: Suppose $\mathbf{C}_1 = \mat{C}(h_1), \mathbf{C}_2 = \mat{C}(h_2) \in \mathcal{C}$. For any $\mathbf{C}_0 = \alpha\mathbf{C}_1 + (1-\alpha) \mathbf{C}_2$, by linearity of expectation, we have that $\mathbf{C}_0 = C(h_0)$ where $h_0 = \alpha h_1 + (1-\alpha) h_2$ .
\end{itemize}

\subsection{Proof of Lemma~\ref{lem:supporting} and Lemma~\ref{lemma:only}}
The proof of Lemma~\ref{lem:supporting} is primarily geometric, and utilizes the following lemma characterizing supporting hyperplanes of convex sets.
\begin{lemma}[Supporting Hyperplane~\citep{rockafellar2015convex}]\label{lem:base_supporting}
  Let $\mathcal{S}$ be a compact convex set, then for every $\mathbf{s} \in \partial\mathcal{S}$ there exists a supporting hyperplane which intersects with $\mathcal{S}$ at $\mathbf{s}$.
\end{lemma}
As a result, given $\mathbf{a}$ as the normal of hyperplane associated with a point $\mathbf{s}*\in \partial\mathcal{S}$, if follows that:
\begin{equation*}
\mathbf{s}* \in \underset{s \in \mathcal{S}}{\text{argmax}} \dprod{\mathbf{a}, \mathbf{s}}
\end{equation*}
Thus, $\mathcal{S}$ has a dual representation, as the intersection of all the half-spaces associated with its supporting hyperplanes.

\begin{proof}[Proof of Lemma~\ref{lem:supporting}]
By the compactness of $\cC$, there exists a $\mat{C}^*$ such that $\mat{C}^* = \arg\max_{\mat{C} \in \cC} \Psi(\mat{C})$. Hence there exists $\bh^* \in \cH_r$ such that $\mat{C}^* = \mat{C}(\bh^*)$. Equivalently, $\bh^* = \arg\max_{\bh\in \cH_r} \cU(\bh)$.
  Let $\mathbf{C}^* \in \partial \mathcal{C}$, then this implies that $\exists \mathbf{L}$ such that:
  $$\mathbf{C}^* = \mathbf{C}(h^*) \in  \argmax_{\mathbf{C} \in \mathcal{C}} \dprod{\mathbf{L}, \mat{C}},$$
  where $h^* \in \argmax_{h \in \mathcal{H}} \dprod{\mathbf{A}^*, \mat{C}(h)}$.
  When $\mathbb{P}$ satisfies Assumption~\ref{assumption1}, this is equivalent to a linear utility metric. For this case, \cite{narasimhan2015consistent} have shown that the max classifier is Bayes optimal almost everywhere.
\end{proof}

\begin{proof}[Proof of Lemma~\ref{lemma:only}]
$\argmax_{\mathbf{C} \in \mathcal{C}} \dprod{\mathbf{L}, \mat{C}}$ is the optimization of a linear function over a compact convex set $\cC$, thus the maximum $\mat{C}^*$ is necessarily achieved at the boundary $\mat{C}^* \in \partial \cC$.
\end{proof}

\subsection{Proof of Lemma~\ref{lem:gradient_support}}
When $\Phi$ is a monotonic function of $\mat{C}$, then $\mat{C}^*$ necessarily lies on $\partial \mathcal{C}$. By Lemma~\ref{lemma:only}, we know that the Bayes optimal follows the weighted form.
By KKT conditions~\citep{boyd2004}, when $\mat{C}^*$ is optimal, the supporting hyperplane must equal the negative gradient of the metric at $\mat{C}^*$. The remainder follows from the scale and shift invariance of the loss matrix.

\subsection{Proof of Theorem~\ref{thm:max_over_eta}}
Let $\mat{L}$ be as defined in Lemma~\ref{lem:supporting}.
By definition in Eq.~\eqref{eq:defC}, we have
\bas{
\max_{\bC} \dprod{\mat{L}, \mat{C}} =& \max_{\bC} \sum_{m,k,\ell} \mat{L}_{m,k,\ell} \mat{C}_{m,k,\ell}\\
=&  \max_{\bh} \sum_{m,k,\ell} \mat{L}_{m,k,\ell} \int_x \bh_k^m(x) \mathbb{\eta}^m_\ell(x) d\bP(x)
}
Note the above maximization is decomposable with respect to $x$ given $\mathbb{\eta}(x)$. For any given $x$, to maximize $\sum_{m,k,\ell} \mat{L}_{m,k,\ell} \bh_k^m(x) \mathbb{\eta}^m_\ell(x) $ subject to the constraint that $\sum_k \bh_k^m(x)=1, \bh_k^m\ge 0$ is equivalent to finding the largest index $k^*$ such that $k^* = \arg\min \dprod{\mat{L}_{m,k,\cdot}, \mat{\eta}^m} $, and set $h_{k^*}^m=1$ and 0 otherwise. This completes the proof.

\section{Proofs for averaging metrics}
\begin{proof}[Proof of Proposition~\ref{prop:micro_share_plane}]
By definition,
\bas{
\Psi_{\text{micro}}(\mat{ \widehat{C} }) = \psi \left(  \widehat{\bC} \bigcdot_1 \frac{1}{M}\bfone_M \right)
}
Taking derivative with respect to $\mat{C}_{m,k,\ell}$, by the chain rule we have
\bas{
\nabla_{m,k,\ell} \Psi( \mat{C} ) =& \nabla_{k,\ell}\psi\left( \frac{1}{M}\sum_{m=1}^M \mat{C}_{m,\cdot,\cdot}\right) e_k\otimes e_\ell \otimes \frac{1}{M}\bfone_M,
}
where  $e_k$ is the $k\th$ standard normal vector.
Hence,
\bas{
\nabla \Psi(\mat{C})=& \nabla \psi\left( \frac{1}{M}\sum_{m=1}^M \mat{C}_{m,\cdot,\cdot}\right) \otimes \frac{1}{M}\bfone_M
}
By Lemma~\ref{lem:gradient_support}, the supporting hyper-plane is $-\nabla \psi\left( \frac{1}{M}\sum_{m=1}^M \mat{C}^*_{m,\cdot,\cdot}\right) \otimes \frac{1}{M}\bfone_M$. By its formulation, we know each slice along the 3rd dimension is the same matrix $-\nabla \psi\left( \frac{1}{M}\sum_{m=1}^M \mat{C}^*_{m,\cdot,\cdot}\right)$, the claim is proved.
\end{proof}

\section{Proof of Proposition~\ref{macro-decompose}}
\begin{proposition}[Decomposability of macro-averaged Bayes Optimal]
  The macro averaged utility decomposes as $\mathcal{U}_\text{macro}(\vect{h}) = \sum_{m\in [M]} \mathcal{U}_m(h_m)$, and the Bayes optimal macro-averaged classifier decomposes as:
  $h_m^*(x) = \underset{h_m}{\argmax} \; \mathcal{U}_m(h_m)$.
\end{proposition}

\paragraph{Proof}
The proof follows by definition.
The confusion matrix at the population level, instead of calculating from the samples, is replaced by the expectation:
\begin{align*}
(\mat{C}_m)_{i,j}(\vect{h})
&= \expect{(\sConf_m)_{i,j}(\vect{h})}\\
&= \frac{1}{N}\sum_{n=1}^{N}\expectx{n}{\widehat{\mat{C}}_{m,i,j}^{(n)}(h_m)} \\
&= \prob{Y_m=i, \hypo_m(X)=j}
\end{align*}
and the utility at population level is given by
\[
\mathcal{U}_{\text{macro}}(\vect{h}) = \frac{1}{M}\sum_{m=1}^{M}\perf\left(\vect{\conf}_m(h_m)\right)
\]

Specifically, the utility of loss-based performance metric is given by
\begin{align*}
\mathcal{U}_\text{macro}(\vect{h})
&= \frac{1}{M}\sum_{m=1}^{M}\Psi^{\mat{L}^m}(\mat{C}_m(h_m))\\
&= \frac{1}{M}\sum_{m=1}^{M}\left[1 - \dprod{\mat{\loss}^m, \mat{\conf}_m(h_m)}\right]\\
&= 1 - \frac{1}{M}\sum_{m=1}^{M}\left[\dprod{\mat{\loss}^m, \mat{\conf}_m(h_m)}\right]
\end{align*}

\section{Proof of Theorem~\ref{consistency-theorem}}

\begin{theorem}[$\perf$-regret of bisection based algorithm]
($\perf$-regret of bisection based algorithm). Let $\psi:\left[0, 1\right]^{K\times K} \to \mathbb{R}_+$ be $\psi(\mat{C}) = \frac{\dprod{\mat{A}, \mat{C}}}{\dprod{\mat{B}, \mat{C}}}$, where $\mat{A}$, $\mat{B} \in \mathbb{R}^{K\times K}_{+}$, $\sup_{(\mat{C}) \in \mathcal{C}_\mathbb{P}} \psi{(\mat{C})} \leq 1$, and $\min_{\mat{C}\in \mathcal{C}_\mathbb{P}}\dprod{\mat{B}, \mat{C}} \geq b$, for some $b>0$. Let $\mathcal{S} = (\mathcal{S}_1, \mathcal{S}_2) \in {(\mathcal{X} \times \mathcal{Y})}^N$ be a training set drawn i.i.d from a distribution $\mathbb{P}$, where $\mathcal{Y} = \left[K\right]^M$. Let $\widehat{\vect{\eta}}: \mathcal{X} \to [\Delta_K]^M$ be the model learned from $\mathcal{S}_1$ in Algorithm~\ref{algo:bisection} and $\vect{h}^{BS}: \mathcal{X} \to \mathcal{Y}$ be the classifier obtained over $\kappa N$ iterations. Then for any $\delta \in (0, 1]$, with probability at least $1-\delta$ (over draw of $\mathcal{S}$ from $\mathbb{P}^{N}$), we have
\begin{align*}
  \util ^*_\mathbb{P} &- \util _\mathbb{P} {[\vect{h}^{BS}]} \leq \frac{2\tau}{m} \sum_{m} \expectx{X}{\norm{\widehat{\vect{\eta}}^m(X) - {\vect{\eta}^m(X)}}_{1}}
  + \\
  & 2 \sqrt{2}C_{\tau} \sqrt{\frac{K^2 \log(K)\log(MN) + \log(K^2/\delta)}{MN}} + 2^{-\kappa MN}
\end{align*}
where $\tau = \frac{1}{b}(||A||_{1} + ||B||_{1})$ and $C>0$ is a distribution-independent constant.
\end{theorem}

\paragraph{Proof}
We prove by exploring the equivalent multi-class classification under multi-output classification paradigm. Let $Z=[X, O]$ denotes a new instance space $\mathcal{Z}$ that adds a feature $m$ to $\mathcal{X}$ space, with $\prob{Z}=\prob{X} \prob{O}$ and uniform label distribution that $\prob{O = m}=\frac{1}{m}$. Then we construct a multi-class classification with $NM$ instances where $\mathcal{Z} \to \Delta_K$.

According to definition, we have new label space, classifier $f(z)$, conditional distribution $\gamma(z)$ and marginal distribution $\prob{Z}$ as
\begin{align*}
&Y_{z} = Y_{m}(x) \in [K] \\
&f(z) = h_{m}(x) \in [K]\\
&\vect{\gamma}(z) = \vect{\eta}^{m}(x)\\
&\prob{z} = \frac{1}{m}\prob{x}\\
\end{align*}
The confusion matrix $D(f)$ is
\begin{align*}
D_{i,j}(f) &= \mathbb{E}_z \mathbb{E}_{y,z}(\ind{y=i, f(z)=j})\\
    & =\frac{1}{m} \sum_{m=1}^M \mathbb{E}_x \mathbb{E}_{y,x} (\ind{y_m=i, h_m(x)=j})\\
    & = \frac{1}{m} C_{i,j}(h_m)
\end{align*}
And the optimal classifier $f^*(z)$ is
\begin{align*}
f^*(z) &= \argmin_{k\in[K]} \vect{\gamma}(z)^T L_k = \argmin_{k\in[K]} (\vect{\eta}^m)^T{(x)} L_k \\
\end{align*}
Then, the multi-class classification $\mathcal{Z} \to \Delta_K$ is equivalent to the original multi-output classification ${(\mathcal{X} \times \left[K\right]^M)}^N \to [\Delta_K]^M$. By Theorem 17 in ~\cite{narasimhan2015consistent}, we have
\begin{align*}
  & \util^*_\mathbb{P} - \util_\mathbb{P} {[\vect{h}^{BS}]} \leq 2\tau \expectx{Z}{\norm{\widehat{\vect{\gamma}}(Z) - {\vect{\gamma}(Z)}}_{1}} \\
    & + 2 \sqrt{2}C_{\tau} \sqrt{\frac{K^2 \log(K)\log(MN) + \log(K^2/\delta)}{MN}} + 2^{-\kappa MN},
\end{align*}
Also note that
\begin{align*}
\expectx{Z}{f(Z)} &= \mathbb{E}_m\expectx{X}{f([X, m])}\\
      & = \frac{1}{m}\sum_m \expectx{x}{\norm{\widehat{\vect{\eta}}^m(X) - \vect{\eta}^m(X) }_1}
\end{align*}
Then finally, we have
\begin{align*}
  & \util^*_\mathbb{P} - \util_\mathbb{P} {[\vect{h}^{BS}]} \leq \frac{2\tau}{m} \sum_{m} \expectx{X}{\norm{\widehat{\vect{\eta}}^m(X) - {\vect{\eta}^m(X)}}_{1}} \\
    & + 2 \sqrt{2}C_{\tau} \sqrt{\frac{K^2 \log(K)\log(MN) + \log(K^2/\delta)}{MN}} + 2^{-\kappa MN},
\end{align*}

\section{Bisection Method - Algorithm}

\begin{algorithm}[h]
  \caption{Bisection Method (for micro averaging of multi-class fractional linear metrics)}
  \label{algo:bisection}
  \begin{tabular}{{@{}ll@{}}}
    \textbf{Input}: & $\mathcal{S} = \set{x^{(n)}, \vect{y}^{(n)}}_{n=1}^N\in(\mathcal{X}\times\range{K}^M)^N$\\
    & $\Psi(\mat{C})=\frac{\dprod{\mat{A}, \mat{C}}}{\dprod{\mat{B}, \mat{C}}}$ where $\mat{A}, \mat{B}\in\mathbb{R}^{K\times K}$.
  \end{tabular}

  \textbf{Parameter}: $\kappa\in\mathbb{N}$

  \begin{algorithmic}[1]
    \STATE Split $\mathcal{S}$ into $\mathcal{S}_1$ and $\mathcal{S}_2$ with size $\left\lceil\frac{N}{2}\right\rceil$ and $\left\lfloor\frac{N}{2}\right\rfloor$; estimate $\widehat{\vect{\eta}}=\set{\widehat{\vect{\eta}}^m}_{m=1}^M$ using $\mathcal{S}_1$
    \STATE Initialize $\vect{h}^0: \mathcal{X}\to\range{K}^M$, $\alpha = 0$, $\beta = 1$
    \FOR{$t = 1$ to $\kappa N$}
    \STATE $\gamma^t = (\alpha^{t-1} + \beta^{t-1})/2$
    \STATE $\widehat{\mat{L}}^t = \gamma^t\mat{B} - \mat{A}$, scaled and shifted to $\range{0, 1}^{K\times K}$.
    \STATE Define $\vect{\widehat{g}}^t = \set{\widehat{g}^t_m}_{m=1}^M$ where ${\widehat{g}^t_m(x)} \in {\argmin_{k\in\range{\classNum}}(\vect{\widehat{L}}_k^t)^\intercal\vect{\widehat{\eta}}^m(x)}$
    \IF{$\widehat{\mathcal{U}}^{\mathcal{S}_2}_\text{micro}(\mathbf{\widehat{g}}^t)\geq\gamma^t$}
    \STATE $\alpha^t = \gamma^t$, $\beta^t = \beta^{t-1}$, $\vect{h}^t = \vect{\widehat{g}}^t$
    \ELSE \STATE $\alpha^t = \alpha^{t-1}$, $\beta^t = \gamma^{t}$, $\vect{h}^t = \vect{h}^{t-1}$
    \ENDIF
    \ENDFOR
  \end{algorithmic}
  \textbf{Output}: $\vect{h}^{(\kappa N)}$
\end{algorithm}

\vspace{-0.1cm}
\subsection{Consistency of Bisection Algorithm}
\label{consistency}

Consistency is a desirable property for a classifier, as it suggests that the procedure has good large sample statistical properties.

\begin{theorem}
\label{consistency-theorem}
($\perf$-regret of bisection based algorithm). Let $\psi:\left[0, 1\right]^{K\times K} \to \mathbb{R}_+$ be $\psi(\mat{C}) = \frac{\dprod{\mat{A}, \mat{C}}}{\dprod{\mat{B}, \mat{C}}}$, where $\mat{A}$, $\mat{B} \in \mathbb{R}^{K\times K}_{+}$, $\sup_{(\mat{C}) \in \mathcal{C}_\mathbb{P}} \psi{(\mat{C})} \leq 1$, and $\min_{\mat{C}\in \mathcal{C}_\mathbb{P}}\dprod{\mat{B}, \mat{C}} \geq b$, for some $b>0$. Let $\mathcal{S} = (\mathcal{S}_1, \mathcal{S}_2) \in {(\mathcal{X} \times \mathcal{Y})}^N$ be a training set drawn i.i.d from a distribution $\mathbb{P}$, where $\mathcal{Y} = \left[K\right]^M$. Let $\widehat{\vect{\eta}}: \mathcal{X} \to [\Delta_K]^M$ be the model learned from $\mathcal{S}_1$ in Algorithm~\ref{algo:bisection} and $\vect{h}^{BS}: \mathcal{X} \to \mathcal{Y}$ be the classifier obtained over $\kappa N$ iterations. Then for any $\delta \in (0, 1]$, with probability at least $1-\delta$ (over draw of $\mathcal{S}$ from $\mathbb{P}^{N}$), we have
$
  \util ^*_\mathbb{P} - \util _\mathbb{P} {[\vect{h}^{BS}]} \leq \frac{2\tau}{m} \sum_{m} \expectx{X}{\norm{\widehat{\vect{\eta}}^m(X) - {\vect{\eta}^m(X)}}_{1}}
  + 2 \sqrt{2}C_{\tau} \sqrt{\frac{K^2 \log(K)\log(MN) + \log(K^2/\delta)}{MN}} + 2^{-\kappa MN},
$
where $\tau = \frac{1}{b}(||A||_{1} + ||B||_{1})$ and $C>0$ is a distribution-independent constant.
\end{theorem}

\section{Additional Discussion of Averaged Multioutput Metrics}
\begin{table*}[th]
  \begin{center}
    \caption{Confusions and performance metrics for three averagings}
    \label{table:perf}
    \begin{tabular}{lcc}
      \toprule
      Averaging & Confusion & Performance Metric\\
      \midrule
      Micro-averaging & $\sConf_{i,j}(\vect{h}) = \frac{1}{MN}\sum_{m=1}^{M}\sum_{n=1}^{N}\sConf_{m,i,j}^{(n)}(h_m)$ &  $\Psi_{\text{micro}}(\mathcal{A}(\vect{h})) = \psi\left(\vect{\sConf}(\vect{h})\right) $\\
      Instance-averaging  &$ (\sConf_n)_{i,j}(\vect{h}) = \frac{1}{M}\sum_{m=1}^{M}\sConf_{m,i,j}^{(n)}(h_m)$ &
      $\Psi_{\text{instance}}(\mathcal{A}(\vect{h})) = \frac{1}{N}\sum_{n=1}^{N}\psi\left(\vect{\sConf}_n(\vect{h})\right)$ \\
      Macro-averaging   & $(\sConf_m)_{i,j}(\vect{h}) = \frac{1}{N}\sum_{n=1}^{N}\sConf_{m,i,j}^{(n)}(h_m)$ &
      $\Psi_{\text{macro}}(\mathcal{A}(\vect{h})) = \frac{1}{M}\sum_{m=1}^{M}\psi\left(\vect{\sConf}_m(h_m)\right)$\\
      \bottomrule
    \end{tabular}

  \end{center}
\end{table*}

The most common technique for constructing multioutput metrics is by averaging multiclass performance metrics, which corresponds to particular settings of $\mathcal{A}(\vect{h})$. Averaged multiclass metrics are constructed by averaging with respect to outputs (instance-averaging), with respect to examples separately for each output (macro-averaging), or with respect to both outputs and examples (micro-averaging). The confusions and performance metrics for micro-, macro- and instance-averaging are straightforward to derive from their definitions, and are as shown in Table~\ref{table:perf}.

Now we turn our attention to characterizing the Bayes optimal classifiers for averaged multioutput metrics. Our first observation is that micro-averaging and instance-averaging, while seemingly quite different in terms of samples, are in fact equivalent as population metrics.
Note that our definitions of population metrics directly follow from the multilabel classification definitions established by ~\citet{NIPS2015_5883}

\begin{proposition}[Micro- and Instance-averaging are equivalent at population level]\label{micro-instance-equivalent}
   Given $\Psi$, for any $\vect{h}$, $\mathcal{U}_{\text{instance}}(\vect{h}) = \mathcal{U}_{\text{micro}}(\vect{h}) = \Psi\left(\frac{1}{M}\sum_{m=1}^M \mat{C}(h_m)\right),$ and consequently, $\vect{h}^*_{\Psi_{\text{instance}}} = \vect{h}^*_{\Psi_{\text{micro}}}$
\end{proposition}

\paragraph{Proof}
  For micro-averaging, at the population level, instead of calculating the confusion from the samples, we replace it{} by the expectation of the confusions:
\begin{align*}
C_{i,j}(\vect{h})
&= \expect{\sConf_{i,j}(\vect{h})}\\
&= \frac{1}{M}\sum_{m=1}^{M}\expectx{n}{\sConf_{m,i,j}^{(n)}(h_m)}\\
&= \frac{1}{M}\sum_{m=1}^{M}\prob{Y_m=i, \hypo_m(X)=j}
\end{align*}
and the performance at population level is given by
\[
\mathcal{U}_{\text{micro}}(\vect{h}) = \psi\left(\vect{\conf}(\vect{h})\right)
\]

Similarly, for instance-averaging, we replace the sample confusion matrix by its expectation:
\begin{align*}
(C_{n})_{i,j}(\vect{h})
&= \expect{(\sConf_n)_{i,j}(\vect{h})}\\
&= \frac{1}{M}\sum_{m=1}^{M}\expectx{n}{\sConf_{m,i,j}^{(n)}(h_m)}\\
&= \frac{1}{M}\sum_{m=1}^{M}\prob{Y_m=i, \hypo_m(X)=j}
\end{align*}
Since $(\mat{C}_{n})(\vect{h})$ is the same for every $n$, we can write it as $\mat{C}(\vect{h})$ instead. The performance at population level is then given by
\begin{align*}
\mathcal{U}_{\text{instance}}(\vect{h})
&= \frac{1}{N}\sum_{n=1}^{N}\psi\left(\vect{\conf}(\vect{h})\right)\\
&= \psi\left(\vect{\conf}(\vect{h})\right)
\end{align*}

For loss-based performance metric, define $\mat{C}_m(h_m) = \expectx{n}{\mat{\sConf}_{m}^{(n)}(h_m)}$, the utility of $\Psi^{\mat{L}}$ is given by
\begin{align*}
\mathcal{U}_\text{micro}(\vect{h}) = \mathcal{U}_\text{instance}(\vect{h})
&= \psi^{\mat{L}}(\mat{C}(\vect{h}))\\
&= 1 - \dprod{\mat{\loss}, \mat{\conf}(\vect{h})}\\
&= 1 - \dprod{\mat{\loss}, \frac{1}{M}\sum_{m=1}^{M}\mat{C}_m(h_m)}\\
&=  1 -  \frac{1}{M}\sum_{m=1}^{M}\dprod{\mat{\loss}, \mat{C}_m(h_m)}
\end{align*}

Therefore, the optimal classifier $\vect{h}^*$ that maximize $\mathcal{U}$ is also the one that minimize $\frac{1}{M}\sum_{m=1}^{M}\dprod{\mat{\loss}, \mat{C}_m(h_m)}$. This is equivalent to finding the minimum of each of $\dprod{\mat{\loss}, \mat{C}_m(h_m)}$ independently.

Similar to \cite{narasimhan2015consistent}, the optimal classifier $\vect{\hypo}^*:\instS\to\left[K\right]^M$ satisfies
\[
h^{*}_m(x) \in \argmin_{k\in\range{\classNum}}\vect{L}_k^\intercal\vect{\marg}^m(x)
\]
and we call $h^{*}$ the $\Psi^\mat{L}$ optimal classifier.

\section{Experiments Details and More Results}
 We report the full results for micro- and macro-averaging on benchmark datasets in Table~\ref{table-uci-start-appendix}. The algorithms evaluated are: (1) multi-output logistic regression classifier for each label (LogReg), (2) k-nearest neighbor with k as 5 (KNN), (3) Random Forest with max depth as 3 (RF), (4) Decision Tree with max depth as 3 (DT), (5) consistent logistic regression (C-LogReg), (6) consistent Decision Tree (C-DT). (1) and (5) are linear and the rest are non-linear classifiers. The hyper-parameters are chosen using double-loop cross-validation.

 As expected, C-DT or C-LogReg always gives the best performance. The difference between C-DT and C-LogReg comes from their consumption of class probabilities from different base learners.

    \begin{table*}[htb]
    \caption{Comparison of performance for multi-output over
    logistic regression (LogReg), k-nearest neighbors (KNN), Random Forests (RF), Decision Tree (DT), consistent multi-output logistic regression classifier (C-LogReg), consistent Decision Tree classifier (C-DT) by micro- and macro-averaging under Ordinal, Micro-F1 and Weighted$_{\bf{\frac{1}{2}}}$ performance metrics. Datasets come from UCI Machine Learning Repository and all results average over 100 iterations with 80\%-20\% train-test split. The last two columns are the proposed post-processing which we found always improved performance.
    }
    \label{table-uci-start-appendix}
    \begin{adjustbox}{center}
        \begin{small}
        \begin{sc}
        \begin{tabular}{|l|c|c|c|c|c|c|}
        \hline
        \rule{0pt}{13pt} Dataset & LogReg & KNN & RF & DT & C-LogReg & C-DT\\[5pt]
        \hline
        \multicolumn{7}{|c|}{\textnormal{Comparison of performance on \textbf{Ordinal} metric for six algorithms on benchmark datasets by \textbf{Micro-averaging}}} \\
        \hline
        Car    & 0.6038$\pm$0.0128 & 0.5799$\pm$0.0126 & 0.6112$\pm$0.0169 & 0.6170$\pm$0.0228 &  \B 0.6962$\pm$0.0075 & 0.6961$\pm$0.0076\\

        Nursery  & 0.7047$\pm$0.0090 & 0.6223$\pm$0.0035 & 0.7084$\pm$0.0148 & 0.7173$\pm$0.0132 & \B 0.7671$\pm$0.0032 & 0.7666$\pm$0.0030\\

        CMC    & 0.7777$\pm$0.0109 & 0.7443$\pm$0.0100 & 0.7644$\pm$0.0123 & 0.7654$\pm$0.0137 & \B 0.7918$\pm$0.0098 & 0.7818$\pm$0.0097\\

        Phish    & 0.7967$\pm$0.0087 & 0.7787$\pm$0.0103 & 0.7888$\pm$0.0157 & 0.7981$\pm$0.0096 & 0.8016$\pm$0.0084 & \B 0.8070$\pm$0.0089\\

        Student    & 0.7704$\pm$0.0122 & 0.7585$\pm$0.0130 & 0.7778$\pm$0.0119 & 0.7792$\pm$0.0119 & 0.7797$\pm$0.0122 & \B 0.7889$\pm$0.0122\\

        \hline
        \multicolumn{7}{|c|}{\textnormal{ Comparison of performance on \textbf{Micro-F1} metric for six algorithms on benchmark datasets by \textbf{Micro-averaging}}} \\
        \hline
        Car    & 0.2775$\pm$0.0191 & 0.1573$\pm$0.0137 & 0.2784 $\pm$ 0.0231 & 0.2941$\pm$0.0194 & \B 0.3322$\pm$0.0175 &  0.3312$\pm$0.0187\\

        Nursery  & 0.4836$\pm$0.0084 & 0.2896$\pm$0.0067 & 0.4531$\pm$0.0307 & 0.4815$\pm$0.0159 & 0.4961$\pm$0.0074 & \B 0.5116$\pm$0.0070\\

        CMC    & 0.4928$\pm$0.0200 & 0.4155$\pm$0.0195 & 0.4870$\pm$0.0176 & 0.4737$\pm$0.0188 & \B 0.4959$\pm$0.0189 & 0.4767$\pm$0.0199\\

        Phish    & 0.6931$\pm$0.0175 & 0.6860$\pm$0.0182 & 0.6868$\pm$0.0226 & 0.7029$\pm$0.0168 & 0.6941$\pm$0.0182 & \B0.7040$\pm$0.0191\\

        Student    & 0.2291$\pm$0.0232 & 0.2240$\pm$0.0273 & 0.2299$\pm$0.0286 & 0.2413$\pm$0.0276 & 0.2567$\pm$0.0247 & \B 0.2875$\pm$0.0269\\

        \hline
        \multicolumn{7}{|c|}{\textnormal{ Comparison of performance on \textbf{Weighted$_{\bf{\frac{1}{2}}}$} metric for six algorithms on benchmark datasets by \textbf{Micro-averaging}}} \\
        \hline
        Car    & 0.1023$\pm$0.0124 & 0.0754$\pm$0.0095 & 0.0526$\pm$0.0094 & 0.0792$\pm$0.0245 & \B0.2487$\pm$0.0152 & 0.2482$\pm$0.0153\\

        Nursery  & 0.2218$\pm$0.0144 & 0.2203$\pm$0.0045 & 0.2361$\pm$0.0172 & 0.2405$\pm$0.0253 & 0.3562$\pm$0.0060 & \B 0.3563$\pm$0.0060\\

        CMC    & 0.1206$\pm$0.0098 & 0.1228$\pm$0.0095 & 0.1124$\pm$0.0115 & 0.1151$\pm$0.0109 & \B 0.1947$\pm$0.0118 & 0.1922$\pm$0.0119\\

        Phish    & 0.2260$\pm$0.0164 & 0.2270$\pm$0.0169 & 0.2286$\pm$0.0211 & 0.2541$\pm$0.0194 & \B 0.3151$\pm$0.0210 & 0.3139$\pm$0.0208\\

        Student    & 0.3073$\pm$0.0195 & 0.3048$\pm$0.0164 & 0.3544$\pm$0.0190 & 0.3344$\pm$0.0224 & \B0.3721$\pm$0.0209 & 0.3716$\pm$0.0215\\

        \hline
        \multicolumn{7}{|c|}{\textnormal{ Comparison of performance on \textbf{Ordinal} metric for six algorithms on benchmark datasets by \textbf{Macro-averaging}}} \\
        \hline
        Car    & 0.6060$\pm$0.0127 & 0.5788$\pm$0.0107 & 0.6131$\pm$0.0176 & 0.6185$\pm$0.0219 & 0.6960$\pm$0.0078 & \B0.6972$\pm$0.0085\\

        Nursery  & 0.7043$\pm$0.0088 & 0.6230$\pm$0.0036 & 0.7022$\pm$0.0145 & 0.7164$\pm$0.0141 & \B 0.7679$\pm$0.0028 & 0.7672$\pm$0.0028\\

        CMC    & 0.7748$\pm$0.0104 & 0.7445$\pm$0.0095 & 0.7643$\pm$0.0119 & 0.7642$\pm$0.0146 & \B 0.7911$\pm$0.0087 & 0.7804$\pm$0.0083\\

        Phish    & 0.7977$\pm$0.0095 & 0.7871$\pm$0.0150 & 0.7871$\pm$0.0150 & 0.7983$\pm$0.0094 & 0.8015$\pm$0.0096 & \B0.8065$\pm$0.0093\\

        Student    & 0.7698$\pm$0.0119 & 0.7581$\pm$0.0118 & 0.7761$\pm$0.0136 & 0.7778$\pm$0.0132 & 0.7785$\pm$0.0111 & \B0.7886$\pm$0.0114\\

        \hline
        \multicolumn{7}{|c|}{\textnormal{ Comparison of performance on \textbf{Micro-F1} metric for six algorithms on benchmark datasets by \textbf{Macro-averaging}}} \\
        \hline
        Car    & 0.2759$\pm$0.0163 & 0.1562$\pm$0.0153 & 0.2762$\pm$0.0191 & 0.2939$\pm$0.0184 & \B 0.3336$\pm$0.0167 & 0.3289$\pm$0.0166\\

        Nursery  & 0.4862$\pm$0.0112 & 0.2784$\pm$0.0072 & 0.4510$\pm$0.0333 & 0.4763$\pm$0.0241 & 0.5052$\pm$0.0072 & \B 0.5190$\pm$0.0074\\

        CMC    & 0.4898$\pm$0.0210 & 0.4113$\pm$0.0195 & 0.4837$\pm$0.0206 & 0.4722$\pm$0.0211 & \B 0.4909$\pm$0.0188 & 0.4737$\pm$0.0204\\

        Phish    & 0.6853$\pm$0.0165 & 0.6767$\pm$0.0162  & 0.6826$\pm$0.0199 & 0.6938$\pm$0.0177 & 0.6915$\pm$0.0182 & \B 0.7039$\pm$0.0181\\

        Student    & 0.1728$\pm$0.0252 & 0.1887$\pm$0.0267 & 0.1189$\pm$0.0143 & 0.1522$\pm$0.0208 & 0.2444$\pm$0.0218 & \B0.2848$\pm$0.0292\\

        \hline
        \multicolumn{7}{|c|}{\textnormal{ Comparison of performance on \textbf{Weighted$_{\bf{\frac{1}{2}}}$} metric for six algorithms on benchmark datasets by \textbf{Macro-averaging}}} \\
        \hline
        Car    & 0.1027$\pm$0.0111 & 0.0736$\pm$0.0097 & 0.0529$\pm$0.0099 & 0.0788$\pm$0.0238 &  \B0.2483$\pm$0.0121 & 0.2480$\pm$0.0120\\

        Nursery  & 0.2213$\pm$0.0137 & 0.2202$\pm$0.0044 & 0.2340$\pm$0.0171 & 0.2367$\pm$0.0255 & \B 0.3564$\pm$0.0055 & \B 0.3564$\pm$0.0055\\

        CMC    & 0.1187$\pm$0.0108 & 0.1216$\pm$0.0116 & 0.1092$\pm$0.0120 & 0.1111$\pm$0.0125 & \B 0.1938$\pm$0.0130 & 0.1915$\pm$0.0128\\

        Phish    & 0.2259$\pm$0.0134 & 0.2286$\pm$0.0158 & 0.2312$\pm$0.0207 & 0.2564$\pm$0.0166 & \B 0.3175$\pm$0.0163 & 0.3161$\pm$0.0158\\

        Student    & 0.3117$\pm$0.0214 & 0.3047$\pm$0.0209 & 0.3540$\pm$0.0231 & 0.3360$\pm$0.0233 & \B0.3706$\pm$0.0242 & 0.3700$\pm$0.0261\\
        \hline
        \end{tabular}
        \end{sc}
        \end{small}
      \end{adjustbox}
    \vskip -0.1in
  \end{table*}

\end{document}